\documentclass[letterpaper, 10 pt, conference]{ieeeconf}
\IEEEoverridecommandlockouts                  
\usepackage{cite}
\usepackage[pdftex]{graphicx}
\usepackage{amsmath}
\usepackage{amssymb}
\usepackage{amsfonts}
\usepackage{array}
\usepackage{xcolor}
\usepackage{tikz}
\usepackage{subcaption}
\usepackage[bb=dsserif]{mathalpha}
\usepackage{booktabs}
\usepackage{xspace}
\usepackage{balance}
\usepackage{url}

\usepackage[ruled,vlined]{algorithm2e}
\usepackage{multirow}
\definecolor{tappanred}{HTML}{9A3324}
\definecolor{rossorange}{HTML}{D86018}
\definecolor{teal}{HTML}{008080}
\definecolor{crimson}{HTML}{DC143C}
\definecolor{darkorange}{HTML}{FF8C00}
\definecolor{green}{HTML}{A5A508}

\newtheorem{example}{Example}

\newtheorem{theorem}{Theorem}
\newtheorem{problem}{Problem}

\newcommand{\set}[1]{\mathcal{#1}}

\newcommand{\val}[1]{\boldsymbol{#1}}
\newcommand{\sign}{\mathrm{sign}}
\newcommand{\signal}{\sigma}

\newcommand{\reals}{\mathbb{R}}
\newcommand{\integers}{\mathbb{Z}}

\newcommand{\sdomain}{\mathcal{S}}
\newcommand{\timedomain}{\mathbb{T}}

\newcommand{\mat}[1]{\begin{bmatrix} #1 \end{bmatrix}}
\newcommand{\suchthat}{\text{ s.t. }}

\newcommand{\dataset}{\set{D}}
\newcommand{\preferencedata}{\set{PR}}

\newcommand{\robustnesstree}{Robustness Computation Tree\xspace}
\newcommand{\RT}{RCT\xspace}
\newcommand{\wstlrob}{weighted robustness\xspace}

\newcommand{\G}{\square}
\newcommand{\F}{\Diamond}
\newcommand{\U}{\text{\bf{U}}}
\newcommand{\predicate}{\mu}

\newcommand{\lb}{a}
\newcommand{\ub}{b} 

\newcommand{\rob}{\rho}

\newcommand{\weight}{w}
\newcommand{\weightset}{\set{W}}
\newcommand{\wval}{\val{w}}

\newcommand{\wrob}{r}

\title{ \LARGE Safe and Optimal Learning from Preferences via Weighted Temporal Logic with Applications in Robotics and Formula 1}

\author{Ruya Karagulle, Cristian-Ioan Vasile, Necmiye Ozay
\thanks{R. Karagulle and N. Ozay are with the University of Michigan. C.I. Vasile is with Lehigh University. Corresponding author email: ruyakrgl@umich.edu}\thanks{This work is supported in part by NSF Grants TI-2303564 and IIS-2442644.
The authors thank G. A. Cardona for technical discussions. R. Karagulle thanks D. Aksoy, A. Kayar, and E. Bugdayci for F1 discussions.}}

\begin{document}
\maketitle

\begin{abstract} 
Autonomous systems increasingly rely on human feedback to align their behavior, expressed as pairwise comparisons, rankings, or demonstrations. While existing methods can adapt behaviors, they often fail to guarantee safety in safety-critical domains. We propose a safety-guaranteed, optimal, and efficient approach for solving the learning problem from preferences, rankings, or demonstrations using Weighted Signal Temporal Logic (WSTL). WSTL learning problems, when implemented naively, lead to multi-linear constraints in the weights to be learned. By introducing structural pruning and log-transform procedures, we reduce the problem size and recast it as a Mixed-Integer Linear Program while preserving safety guarantees. Experiments on robotic navigation and real-world Formula 1 data demonstrate that the method captures nuanced preferences and models complex task objectives.
\end{abstract}

\section{Introduction}
Autonomous systems are increasingly part of our daily lives, from driverless cars in urban navigation to household robots performing domestic chores. Since these systems operate closely alongside humans, learning from human feedback is a natural way to ensure the system's behaviors align with human desires. Humans express their preferences in many forms: pairwise comparisons (“I prefer A over B”) \cite{shin2023plbenchmarks}, rankings (“I prefer A over B, and B over C”) \cite{Myers2022rank, Brown2022ranked}, or demonstrations (examples of trajectories or task executions) \cite{Schrum2024maveric, deHeuvel2022vrpersonalized}. Preference learning and learning-to-rank methods leverage comparative judgments to form the queries and collect preference data about behaviors. These preferences are then used in frameworks such as preference-based learning \cite{Biyik2022thesis}, reinforcement learning from human feedback \cite{Christiano2017rlfrompref, Hejna2023fewshot}, and direct optimization \cite{An2023dppo}, to tune the system's behavior. In learning from demonstrations, users directly demonstrate the task or guide the system, by tools like teleoperation. After data collection, these demonstrations are fed into frameworks such as behavioral cloning or inverse reinforcement learning. 

While these approaches have shown promise in adapting the system behavior based on preferences, they fall short in providing rigorous safety guarantees when the systems are deployed in safety-critical domains, such as autonomous vehicles and industrial automation. These methods generally assume that users demonstrate and prefer safe options. However, in cases where the user cannot reliably judge safety, this assumption can be dangerous or even fatal for all agents in the environment. The challenge is to enable learning within the space of safe behaviors, even in cases where user preferences conflict with safety requirements.

Karagulle et al. offer a solution for the safe preference learning challenge \cite{Karagulle2024spl, Karagulle2024sapl}. Combining temporal logic formalism with preference learning using pairwise comparisons, their proposed framework learns the importance order of the sub-tasks and time instances using preference data. More specifically, they express system specifications using Weighted Signal Temporal Logic (WSTL) \cite{Mehdipour2021wstl} and learn the weights of the formula. This learned WSTL specification can be used for control synthesis tasks \cite{Karagulle2024cccc}. The framework preserves the qualitative semantics of the task specification, ensuring that the meaning of safety does not depend on learned weights. As WSTL provides a quantitative measure for satisfaction and violation of specifications, the framework guarantees that an unsafe behavior is never favored over a safe one. When the weights of the formula are known, WSTL formulas can be encoded as linear constraints with integers \cite{CaKaVa-NAHS-2025}. However, when weights are to be learned, encoding yields multi-linear constraints, as weights appear multiplicatively. Thus, solving safe preference learning to optimality leads to a Mixed-Integer Problem with multi-linear constraints, a hard class of problems to solve. Instead, the authors use gradient-based learning or random sampling for the learning part. While these methods are effective in finding suitable weights, both methods are fundamentally limited in finding the optimal solution in terms of the number of preferences satisfied. Moreover, gradient-based approaches are susceptible to getting stuck at local minima due to the quantitative semantics' complex nature. 

In this paper, we first extend the safe preference learning problem to richer forms of feedback: pairwise preferences, rankings, and demonstrations. We overcome the computational challenges with two key procedures. We propose the $\log$-transform to linearize constraints and recast the problem as an MILP, proving it preserves the optimizer. Since the $\log$ function requires strictly positive terms, we introduce \textit{structural pruning} to eliminate formula parts that do not affect quantitative semantics. Structural pruning, therefore, provides an efficient encoding of constraints as well. We prove that structural pruning preserves the quantitative semantics of signals over formulas. Together, these techniques allow us to achieve efficient and \emph{optimal} learning, as opposed to heuristic approaches that lack optimality guarantees \cite{Karagulle2024spl}. 

\begin{figure*}[ht]
    \centering
    \begin{subfigure}[t]{0.2\linewidth}
        \centering
 \begin{tikzpicture}[
  level distance=1.5 cm,
  every node/.style={inner sep=0pt, thick, color= darkorange},
  nonleaf/.style={draw, circle, color=teal, minimum size=1cm},
  edge from parent/.style={draw=crimson, thick},
  level 1/.style={sibling distance=40mm},
  level 2/.style={sibling distance=10mm}]

  \node[nonleaf] {$\F_{[3,5]}$}
      child {node[nonleaf] {$\land$}
        child {node at (0,0.2) {$\signal \geq 0$}}
        child {node at (0,-0.2) {$\signal \leq 5$}}
      };
  \end{tikzpicture}
        \caption{AST of $\phi = \F_{[3,5]}(0 \leq \signal \leq 5)$}
        \label{fig:ast}
    \end{subfigure}
\hfill
    \begin{subfigure}[t]{0.35\linewidth}
        \centering
\begin{tikzpicture}[
  level distance=1.5cm,
  sibling distance=20mm,
  every node/.style={inner sep=0pt, thick, color= darkorange},
  nonleaf/.style={draw, circle, color=teal, minimum size=1cm},
  edge from parent/.style={draw=crimson, thick},
  level 2/.style={sibling distance=10mm}]

\node[nonleaf] {$\F_{[3,5]}, 0$}
  child {node[nonleaf] {$\land, 3$}
    child {node at (0,0.2) {$\signal(3) \geq 0$}}
    child {node at (0,-0.2) {$\signal(3) \leq 5$}}
  }
  child {node[nonleaf] {$\land, 4$}
    child {node at (0,0.2) {$\signal(4) \geq 0$}}
    child {node at (0,-0.2) {$\signal(4) \leq 5$}}
  }
  child {node[nonleaf] {$\land, 5$}
    child {node at (0,0.2) {$\signal(5) \geq 0$}}
    child {node at (0,-0.2) {$\signal(5) \leq 5$}}
  };
\end{tikzpicture}
\caption{The \RT at time $t=0$, $\mathcal{T}_{\phi,0}$.}
        \label{fig:rt}
    \end{subfigure}
\hfill
    \begin{subfigure}[t]{0.4\linewidth}
        \centering
\begin{tikzpicture}[level distance=1.5cm,
  sibling distance=20mm,
  every node/.style={inner sep=0pt, thick, color= darkorange},
  nonleaf/.style={draw, circle, color=teal, minimum size=1cm},
  edge from parent/.style={draw=crimson, thick},
  level 2/.style={sibling distance=10mm}
]

\node[nonleaf] {$\F_{[3,5]}, 0$}
  child [edge from parent path={(\tikzparentnode) -- node[midway, left, align=center, font=\scriptsize,  color = crimson] {$\rob_\land^{\signal(3)}$ \\$=-1$} (\tikzchildnode)},
         edge from parent/.style={draw=crimson, thick, opacity=0.3}]
    { node[nonleaf, opacity=0.3] {$\land, 3$}
      child [edge from parent path={(\tikzparentnode) -- node[midway,left, align=center, font=\scriptsize, color = crimson] {$\rob_{\predicate_1}^{\signal(3)}$ \\ $ = -1$} (\tikzchildnode)}]
        { node at (0, 0.2) {$\signal(3) \ge 0$} }
      child [edge from parent path={(\tikzparentnode) -- node[midway,right, align=center, font=\scriptsize, color = crimson] {$\rob_{\predicate_2}^{\signal(3)}$ \\ $ = 6$} (\tikzchildnode)}]
        { node at (0, -0.2) {$\signal(3) \le 5$} }
    }
  child [edge from parent path={(\tikzparentnode) -- node[midway, left, align=center, font=\scriptsize, color = crimson] {$\rob_\land^{\signal(4)}$ \\$=1$} (\tikzchildnode)}]
    { node[nonleaf] {$\land, 4$}
      child [edge from parent path={(\tikzparentnode) -- node[midway, left, align=center, font=\scriptsize, color = crimson] {$\rob_{\predicate_1}^{\signal(4)}$ \\ $ = 4$} (\tikzchildnode)}]
        { node at (0,0.2) {$\signal(4) \ge 0$} }
      child [edge from parent path={(\tikzparentnode) -- node[midway,right, align=center, font=\scriptsize, color = crimson] {$\rob_{\predicate_2}^{\signal(4)}$ \\ $ = 1$} (\tikzchildnode)}]
        { node at (0,-0.2) {$\signal(4) \le 5$} }
    }
  child [edge from parent path={(\tikzparentnode) -- node[midway, right, align=center, font=\scriptsize,  color = crimson] {$\rob_\land^{\signal(5)}$ \\$=2$}  (\tikzchildnode)}]
    { node[nonleaf] {$\land, 5$}
      child [edge from parent path={(\tikzparentnode) -- node[midway, left, align=center, font=\scriptsize, color = crimson] {$\rob_{\predicate_1}^{\signal(5)}$ \\ $ =2$} (\tikzchildnode)}]
        { node at (0,0.2) {$\signal(5) \ge 0$} }
      child [edge from parent path={(\tikzparentnode)-- node[midway,right, align=center, font=\scriptsize, color = crimson] {$\rob_{\predicate_2}^{\signal(5)}$ \\ $ = 3$}  (\tikzchildnode)}]
        { node at (0,-0.2) {$\signal(5) \le 5$} }
    };
\end{tikzpicture}
\caption{Pruned \RT at $t=0$, with the robustness values ($\rob(\signal, \phi_i, t')$ denoted as $\rob_{\phi_i}^{\signal(t')}$) noted at edges. Pruned parts are shown with low opacity.}
        \label{fig:pruned_rt}
    \end{subfigure}
    \caption{Trees associated with $\phi = \F_{[3,5]}(0\leq \signal \leq 5)$; AST, \RT, and pruned \RT for $\signal = [5,6,7,-1,4,2]$, respectively.}
    \label{fig:ast_rt}
    \vspace{-0.5cm}
\end{figure*}

We demonstrate the performance of the proposed method on a safe preference learning problem for a simple robotics navigation task, and a learning-to-rank problem using real-world Formula 1 data. In the robot navigation experiment, we show that the method is responsive to even small changes in preferences and successfully reflects these differences in the synthesized trajectories. In the Formula 1 experiment, we illustrate that the learned weights capture complex structures like race performance, effectively modeling race strategies and providing insights into what factors influence successful outcomes over the course of the race.

\section{Preliminaries}
\subsection{Signal Temporal Logic (STL)}
STL is used to reason about signals $\signal: \timedomain \to \sdomain$, where $\timedomain \subset \integers_{\geq0}$ is the time domain and $\sdomain \subseteq \reals^{m}$ is the $m$-dimensional real-valued signal domain.
A well-formed STL formula $\phi$ follows the grammar $\phi ::= \top \mid \predicate \mid \lnot \phi \mid \phi_1 \land \phi_2 \mid \phi_1 \U_{[\lb,\ub]} \phi_2$. Syntax elements have the following interpretations:  $\top$ is Boolean truth, $\predicate$ is a predicate of the form $\predicate(\signal(t)):= h_\predicate(\signal(t)) \geq 0$ where $h_\predicate: \sdomain \to \reals$ maps the signal value at time instant $t$, to a real value\footnote{Note that it is possible to transform any (in)equality of the form $f(x) \sim c$, where $\sim \in \{\leq, \geq, =\}$, into the form defined in the syntax.}. The operator $\lnot$ is the negation, $\land$ is the conjunction, and $\U_{[\lb,\ub]}$ is the ``Until" operator\footnote{Additional operators; disjunction $\lor$, Always $\G_{[\lb,\ub]}$, and Eventually $\F_{[\lb,\ub]}$ can be derived from the grammar as $\phi_1 \lor \phi_2 = \lnot(\lnot \phi_1 \land \lnot \phi_2)$, $\F_{[\lb,\ub]}\phi = \top \U_{[\lb,\ub]} \phi$, and $\G_{[\lb,\ub]}\phi = \lnot(\F_{[\lb,\ub]}\lnot \phi)$.}. Subscript ${[\lb, \ub]}$ with $\lb, \ub \in \integers_{\geq0}$ defines the time interval that the temporal operator is acting on. When the time interval includes the whole signal length, we omit the time interval subscript. The qualitative semantics of STL are defined in \cite{Donze2010}. If a signal $\signal$ satisfies (violates) a formula $\phi$ at time $t$, it is shown as $(\signal, t)\models \phi$ ($(\signal, t) \not \models \phi$). For qualitative semantics at time $t=0$, we omit $t$ and write $\signal \models \phi$. STL has quantitative semantics to measure how well the signal satisfies the formula or how badly it violates the formula. There are different quantitative semantics in the literature \cite{Donze2010, Varnai2020robustness, MeVaBe-TAC-2025}. We follow the metric defined in \cite{Donze2010}, and call it the \textit{robustness}. The robustness of a signal at time $t$ is shown as $\rob(\signal, \phi, t)$. When $t=0$, we denote it as $\rob(\signal, \phi)$. When we work with finite signals, we assume that the signal length covers the time horizon of the formula, a function of time intervals of temporal operators. Additionally, we use STL over discrete-time (as opposed to continuous-time) finite signals with necessary modifications as needed \cite{DeGiacomo2013ltlf}.
 
Each STL formula can be represented by an Abstract Syntax Tree (AST) \cite{Li2021ast}. In an AST, nodes represent syntactic elements: leaf nodes correspond to predicates or $\top$, non-leaf nodes correspond to logical and temporal operators. Edges encode the relationships between operators and their operands in each subformula. Unary operators have one child node, while binary operators have two. To reason about the robustness more systematically, we represent the robustness computation as a tree. Building on the AST, we define the \robustnesstree (\RT) of $\phi$, at time $t$ as $\mathcal{T}_{\phi,t} = (\mathcal{N}_{\phi,t}, \mathcal{E}_{\phi,t})$, where $\mathcal{N}_{\phi,t}$ is the set of nodes, obtained by extending each AST node with the time index at which the corresponding operator or predicate is evaluated during robustness computation, and $\mathcal{E}_{\phi,t}$ is the set of directed edges encoding parent–child relationships in the time-indexed structure. Placing the signal's predicate robustness values $h_\predicate(\signal)$ at leaf nodes of $\mathcal{T}_{\phi,t}$, and replacing each operator with its function defined in the quantitative semantics, the tree mirrors the robustness computation for $\rob(\signal, \phi, t)$ when evaluating from leaf to root.
For a subformula $\phi_i$ and $t'$ of $\phi$, the subtree $\mathcal{T}_{\phi_i, t'}$ corresponds to $\rho(\signal, \phi_i, t')$.

\begin{example}\label{ex:1}
Let $\phi = \F_{[3,5]}(0 \leq \signal \leq 5)$ be an STL formula. Figure~\ref{fig:ast} and~\ref{fig:rt} show the AST and \RT of $\phi$ at $t=0$, respectively. Tree $\mathcal{T}_{0 \leq \signal \leq 5, t'}$ represents $\rob(\signal, 0 \leq \signal \leq 5, t')$.
\end{example}

\subsection{Weighted Signal Temporal Logic (WSTL)}
WSTL is tailored to represent priorities and preferences in STL formulas \cite{Mehdipour2021wstl}. Its syntax extends STL syntax as 
$\phi ::= \top \mid \predicate \mid \lnot \phi \mid \phi_{1} \land^{\weight} \phi_{2} \mid \phi_{1} \U_{[\lb,\ub]}^{\weight^1,\weight^2} \phi_{2}$,
where the weights are $\weight \in \reals_{>0}^2$ and $\weight^1, \weight^2 : [\lb, \ub]\to \reals_{>0}$. All operators are interpreted as in STL. The quantitative semantics of WSTL is called \textit{\wstlrob}, and defined as \cite{Karagulle2024spl}:
\begin{align}\label{eq:quantsemantics}
\wrob(\signal, \top, t) &= \infty \nonumber\\
\wrob(\signal, \predicate, t) &= \rob(\signal,\predicate,t) \nonumber\\
\wrob(\signal, \lnot \phi, t) &= -\wrob(\signal,\phi, t) \nonumber\\
\wrob(\signal, \phi_1 \land^{\weight} \phi_2, t) &= \min \big (\weight_1\wrob(\signal, \phi_{1},t),  \weight_2\wrob(\signal, \phi_{2},t) \big ) \nonumber\\
\wrob(\signal, \phi_1 \U_{[\lb,\ub]}^{\weight^1,\weight^2} \phi_2, t) &= \hspace{-0.2cm} \max\limits_{t' \in [\lb,\ub]} \hspace{-0.2cm} \Big( \min \big( \weight^1_{t'} \hspace{-0.2cm}  \min\limits_{t'' \in [t,t+t')} \hspace{-0.2cm} \wrob(\signal,\phi_1,t''), \nonumber\\
& \quad \weight^2_{t'}\wrob(\signal, \phi_2, t+t')\big ) \Big )
\end{align}

The derived operators have the following definitions:
\begin{align}\label{eq:quantsemanticsdr}
\wrob(\signal, \phi_1 \lor^{\weight} \phi_2, t) &= \max \big (\weight_1\wrob(\signal, \phi_{1},t),  \weight_2\wrob(\signal, \phi_{2},t) \big ) \nonumber\\
\wrob(\signal, \G^{\weight}\phi, t) &= \min\limits_{t' \in [\lb,\ub]} \big(w_{t'} \wrob(\signal,\phi, t + t')\big) \nonumber\\
\wrob(\signal, \F^{\weight}\phi, t) &= \max\limits_{t' \in [\lb,\ub]} \big(w_{t'} \wrob(\signal,\phi, t + t')\big)
\end{align}
To denote the \wstlrob at $t=0$, we use $\wrob(\signal, \phi)$. As shown in \cite{Karagulle2024spl, Mehdipour2021wstl}, the quantitative semantics in \eqref{eq:quantsemantics} is sound. The \RT can be defined for WSTL formulas as well. Weights of the operators are added as costs to the relevant edges, and we obtain $\mathcal{T}_{\phi_{\wval},t} = (\set{N}_{\phi, t}, \set{E}_{\phi,t}, \wval)$.

\subsection{Parametric Weighted Signal Temporal Logic (PWSTL)}
To enable learning from data, we treat weights as parameters and define an extension to WSTL called Parametric Weighted Signal Temporal Logic (PWSTL) (cf. \cite{Yan2021stlnn}). We denote the set of unknown parameters as $\weightset$, and PWSTL formulas as $\phi_{\weightset}$. The weight set is the set (or a subset) of all possible weights in the formula. A PWSTL formula results in a WSTL formula $\phi_{\weightset=\wval}$ with the valuation $\wval$ of the parameters. For simplicity, if it is clear from the context, we write $\phi_{\wval}$ to denote the WSTL formula with valuation $\wval$. Unless noted otherwise, the weight set of the PWSTL version of an STL formula includes all possible weights.

\section{Problem Statement}\label{sec:pstament}
Learning from human feedback in autonomous systems encompasses problems such as learning from preferences, demonstrations, or rankings. Classical approaches include inverse reinforcement learning and behavioral cloning, which directly infer policies or reward structures from data \cite{Biyik2022learningreward, Griffith2013policyfromhf}. Both methods typically require a reinforcement learning-style framework. In contrast, we propose to learn a utility function that captures the preferences independent of a specific agent or environment. Ideally, this utility function serves as a unifying representation of human intent, and whenever trajectories consistent with these preferences are required, they can be synthesized via downstream trajectory generation methods. To restrict the search space of possible utility functions, we focus on a structured family of functions defined by the safety requirements or task specifications. Particularly, we leverage the expressive power of WSTL to describe specifications and the \wstlrob. By parameterizing weights in the quantitative semantics, we enable different weight valuations to encode potential instantiations of the utility function. The learned WSTL formula can be seamlessly integrated to control synthesis frameworks such as model predictive control (MPC), to generate safe and personalized trajectories. Moreover, as weight valuations represent the importance of subformulas or time instances, the result of the learning framework is interpretable, in contrast to using neural network-like structures for learning.

In other words, solving learning-from-human-feedback problems via WSTL requires a task description written in STL and a user dataset, and searches for weighted quantitative semantics that maximize the problem's objective. Formally, we define the problem as follows:
\begin{problem}[Learning from Human Feedback (LHF)]\label{problem:lhf}
Consider an STL formula $\phi$ and a dataset $\dataset$ obtained from user feedback. LHF aims to find a weight valuation $w^*$ for $\phi_{\weightset}$, the PWSTL version of the formula $\phi$ with $|\weightset|=n$, that maximizes the objective value $f(\phi_{\weight}, \dataset)$. That is, 
    \begin{equation}\label{eq:pref_learning}
        \begin{array}{ccc}
            \weight^* \in & \max\limits_{
            \weight \in \reals_{>0}^n
            } & f(\phi_{\weight}, \dataset) \\
            &\suchthat & \text{dataset constraints}
        \end{array}.
    \end{equation}
\end{problem}
The form of the objective function and dataset constraints depends on the learning setup and the available data. For safe preference learning, the user dataset is given as  $\preferencedata = \{(\signal_{i_1}, \signal_{i_2}, l_i )\}_{i=1}^P$, where $\signal_{i_1}$ and $\signal_{i_2}$ are compared and $l_i \in \{1,-1\}$ denotes its label: $l_i = 1$, when $\signal_{i_1}$ is preferred over $\signal_{i_2}$ and $l_i = -1$ otherwise. The goal is to maximize the number of correctly ordered pairs, leading to the objective function of the form $f(\phi_{\weight}, \preferencedata) = \sum_{i=1}^{P}\mathbb{1}_{\big(l_i\Delta_i(w) > 0\big) } $, and constraints $\Delta_i(w) = r(\signal_{i_1}, \phi_{\weight}) - r(\signal_{i_2}, \phi_{\weight})$. 

For learning to rank, the dataset is an ordered set $\set{R} = \{\signal_{1} \succeq  \signal_{2} \succeq \ldots \succeq \signal_{P}\}$. This can be reformulated as a pairwise preference dataset $\preferencedata = \{(\signal_{i}, \signal_{j}, 1 )\}_{i=1, j=i+1}^{P-1, P}$, reducing the problem to safe preference learning with the same objective function and $\binom{P}{2}$ constraints. Finally, in learning from demonstrations, we are given a set of ``desired" trajectories $\dataset = \{\signal_i\}_{i=1}^D$. In this setup, we can use $f(\phi_\weight, \dataset) = \sum_{i=1}^D \wrob(\signal_i, \phi_\weight)$ as the objective value, which aims to assign high weighted robustness values to desired trajectories. When dataset constraints are chosen as above, we can restrict the search domain for weights from strictly positive reals to a bounded region $a_1<\weight_i \leq a_2,\; i\in[1,n]$, where $0\leq a_1 < a_2$. The proof of feasibility equivalence can be found in Th. 2 \cite{Karagulle2024spl}. In all cases, the cost function may include regularization terms in the form of $f(\phi_{\weight}, \preferencedata)+\lambda|\weight|_p$, where $|\weight|_p$ denotes the $p$-norm of the vector of all weights, with a slight abuse of notation.

The computational complexity of WSTL synthesis problems grows with the complexity of the formula. In general, computing the \wstlrob for a signal involves evaluating predicate values at multiple time instances, multiplying them with corresponding weight valuations, and applying nested \texttt{min/max} operations, following the WSTL semantics. Since weights appear multiplicatively, optimizing over them results in a mixed integer program (MIP) with multi-linear constraints, where integers appear as in \cite{CaKaVa-NAHS-2025}. As such, weight synthesis problems are solved with heuristic approaches such as gradient descent \cite{Leung2023stlcg} or sampling \cite{Karagulle2024spl} in the literature.

\section{Reduction to MILP}

In this paper, we propose two procedures to reduce the problem size and linearize the constraints, making the problem easier to solve. The first procedure, called \textit{structural pruning}, systematically examines how values in the \robustnesstree are propagated to the root node to determine the robustness value of a signal with respect to a formula. Based on this observation, we can prune the branches that cannot affect the robustness value regardless of the weight valuations assigned to them. The second procedure, called the \emph{$\log$-transform}, leverages the product identity of the logarithm function to convert constraints in the multiplication form into constraints in the summation form. Since the logarithm is only defined for positive values, a naive approach requires that all subformulas (subtrees in \RT) must have positive robustness for the transformation to be valid, which is a restrictive requirement in practice. To address this, we combine the $\log$-transform with structural pruning, and ensure that only contributing parts with the same robustness value sign of the formula are considered in the robustness computation, therefore making all variables of the constraints that encode this computation have the same sign. Together, these two procedures allow the problem to be reformulated as an MILP with a potentially reduced number of constraints.

Before describing the procedures in detail, we note that all formulas are assumed to be in positive normal form (PNF), that is, negations appear only in front of predicates. While this assumption is not restrictive, as negations can be pushed to the predicate level, we discuss the rationale later.

\subsection{Structural Pruning}
The key insight of structural pruning is that if the overall robustness of a signal is positive (indicating satisfaction), then any subformula with negative or zero robustness values must be ``absorbed" by one of its parent operations in the \robustnesstree. Consequently, when fitting weights to this tree, we know that weights connecting nodes with negative or zero robustness value to their parents cannot change the robustness. This allows us to systematically prune parts of the tree whose values are not actively contributing to robustness. 

More formally, consider a signal $\signal$, an STL formula $\phi$, and the associated \RT $\mathcal{T}_{\phi, t}$. If $\rob(\signal, \phi) > 0$, then only the subtrees $\mathcal{T}_{\phi_i, t'}$ with positive robustness values, i.e., $\rob(\signal, \phi_i, t') > 0$, can potentially determine the final robustness value at the root node of $\mathcal{T}_{\phi,t}$ when weights are introduced to the formula. Similarly, if $\rob(\signal, \phi) < 0$, then only subtrees with negative robustness values affect the final robustness. Note that if $\rob(\signal, \phi) = 0$, then there does not exist a weight valuation that can change the \wstlrob values. This observation leads to our structural pruning algorithm, which is recursively defined over the \robustnesstree, and shown in Algorithm~\ref{alg:structural_pruning}. 

\begin{algorithm}[!htb]
\caption{Structural Pruning}
\label{alg:structural_pruning}
\SetKwProg{Fn}{Function}{:}{end}
\Fn{Prune($\signal, \phi, t, s_\phi$)}{
    \eIf{Root node of $\set{T}_{\phi,t}$ is an Operator}{
        \For{$(\phi_i,t') \in$ Children of the Root}{
            \eIf{$\sign(\rob(\sigma,\phi_i,t')) = s_\phi$}{
                $\set{T}_{\phi_i,t} \gets$ Prune($\signal, \phi_i, t', s_{\phi_i}$)
            }{
                $\set{T}_{\phi,t} \gets \set{T}_{\phi,t} \setminus \set{T}_{\phi_i,t'}$ 
            }
        }
    }{
        \Return $\set{T}_{\phi,t}$
    }
}
\vspace{-0.2em}
\end{algorithm}

The algorithm is given a signal, denoted $\signal$, a formula, denoted $\phi$, the time instance where the robustness is computed, denoted $t$, and the sign of the robustness value, denoted $s_\phi = \sign(\rob(\signal, \phi, t))$. The procedure starts from the root node of the \RT. If this node is an operator, i.e., not a predicate node, then for all children of the operator (line 3), we check the sign of the robustness value of the signal with respect to that child (line 4). If the sign of the child matches with $s_\phi$, then we recursively prune the child subtree (line 5); otherwise, we remove this child subtree, since it cannot yield a better value than other children in the min/max operation (line 7). Finally, if we reach a predicate node, as it has the same robustness sign as $s_\phi$, we keep that leaf-node (line 11). Note that this algorithm returns at least one node. 

\begin{theorem}\label{theorem:pruning}
    Given a signal $\signal$ and a formula $\phi$, structural pruning preserves the quantitative semantics of STL formulas. That is, the robustness value computed from the pruned \RT is $\rob(\signal,\phi,t)$.
\end{theorem}

The proof follows the fact that for any operator, a child whose robustness has an opposite sign from the parent cannot affect the $\min$ or $\max$ computation. By recursively applying this argument from root to leaves, all removed nodes are non-influential, so the robustness of the pruned formula equals that of the original formula.

\addtocounter{example}{-1}
\begin{example}[continued]
    Consider $\phi$, and $\signal = [5,6,7,-1,4,2]$. Figure~\ref{fig:pruned_rt} shows the pruned \RT, $\mathcal{T}_{\phi, t}^{pruned}$. 
\end{example}

Thus far, we have defined structural pruning for STL formulas. Because WSTL semantics are also sound, meaning weights cannot change the sign of the robustness values, the same procedure can be directly applied to WSTL (and PWSTL) formulas, allowing us to prune irrelevant weights as well. Since robustness values with opposing signs do not contribute to the final \wstlrob value, varying their associated weights cannot change the \wstlrob of the signal. Therefore, these weights are not active decision variables. We can then write the robustness constraints based on the pruned \RT of the WSTL formula.

\subsection{The $\log$-transform of constraints}
The $\log$-transform is motivated by the desire to linearize the constraints related to the robustness computation in weight synthesis problems. Prior work \cite{Raman2014mpcstl, CaKaVa-NAHS-2025} demonstrates that constraints of the form $\rob(\signal, \phi)>0$ can be encoded as a series of mixed-integer linear constraints. In the case of the \wstlrob, however, the temporal and logical operators multiply their operands’ robustness values by the associated weights (see Eq.~\eqref{eq:quantsemantics}). Encoding these operations directly, following \cite{Mehdipour2021wstl, CaKaVa-NAHS-2025}, introduces multi-linearity whenever weights are treated as decision variables.

To address this, we apply a logarithmic transformation to both sides of mixed multi-linear integer constraints and convert the products into sums, thus linearizing them with respect to the weights. While effective in linearizing the product terms, the logarithm is only defined over the positive domain. In its most naive form, this requires both sides of every constraint to be positive. In other words, this transformation is valid only for signals that satisfy all predicates at all time instances that are used in the robustness computation (therefore, in constraints), which is not realistic. Before introducing our method for extending the transform to arbitrary signals (independent of satisfaction or violation), we first define the $\log$-transform of the robustness functions in Eq.~\eqref{eq:quantsemantics} under the assumption that all predicate values are strictly positive. The $\log$-transform of the robustness of signal $\signal$ at time $t$ is defined recursively as: 
\begin{equation}
\label{eq:logr}\arraycolsep=1pt
\begin{array}{rl}
    \log\wrob(\signal, \top, t) &= \infty, \\
    \log\wrob(\signal, \predicate, t) &= \log(h_\predicate(\signal(t))),\\
    \log\wrob(\signal, \phi_1 \wedge^{w} \phi_2, t) &= \min\limits_{i \in \{1,2\}}\big(\log(w_i) + \log\wrob(\signal, \phi_i, t)\big),\\
    \log\wrob(\signal, \phi_1 \U_{[\lb,\ub]}^{w^1,w^2} \phi_2, t)) &= \max\limits_{t' \in[\lb,\ub]} \Big(\min\big(\\ &\log(w^1_{t'}) + \log\wrob(\signal, \phi_2, t+t'),\\
    &\log(w^2_{t'})+\hspace{-0.5em} \min\limits_{t'' \in [t, t+t']}\hspace{-0.5em}\log\wrob(\signal, \phi_1, t'')\big)\Big).\\
\end{array}
\end{equation}
The derived operators are defined as: 
\begin{equation}
\label{eq:logr_derived}\arraycolsep=1pt
\begin{array}{rl}
    \log\wrob(\signal, \phi_1 \vee^{w} \phi_2, t) &= \max\limits_{i \in \{1,2\}}\hspace{-0.3em}\big(\log(w_i) + \log\wrob(\signal, \phi_i, t)\big),\\
    \log\wrob(\signal, \F_{[\lb,\ub]}^{w} \phi, t)&= \max\limits_{t' \in [\lb,\ub]}\hspace{-0.3em}\big(\log(w_{t'}) + \log\wrob(\signal, \phi, t+t')\big),\\
    \log\wrob(\signal, \G_{[\lb,\ub]}^{w} \phi, t)&=\min\limits_{t' \in [\lb,\ub]}\hspace{-0.3em}\big(\log(w_{t'}) + \log\wrob(\signal, \phi, t+t')\big).
\end{array}
\end{equation}

We can still convert the right-hand sides of Eqs.~\eqref{eq:logr} and~\eqref{eq:logr_derived} into a series of mixed-integer linear constraints using \cite{CaKaVa-NAHS-2025}. Note that we do not define the $\log$-transform for the robustness computation of negation. The restriction arises due to the domain in which the logarithm is defined: for a formula $\phi$, if $\log(\wrob(\signal,\phi))$ is defined, then $\log(\wrob(\signal, \lnot \phi))$ is not, and vice versa. In consequence, negation is not allowed in the formulas, which restricts our attention to formulas in PNF only. Negation in front of predicates is not an issue since it can be carried inside predicates by reversing the inequality.

When applying the $\log$-transform to constraints in Problem~\ref{problem:lhf}, $\log(h_\predicate(\signal(t)))$ are computed a priori. The decision variables are $\log(w_i)$. We perform a change of variables $v_i = \log(w_i)$ during optimization, leaving the constraints $\log$-free. We can recover the original weights via $w_i = \exp(v_i)$ afterward. Similarly, the robustness of the weighted formula can be recovered as $\wrob(\signal, \phi, t) = \exp(\log \wrob(\signal, \phi_v, t))$ directly from the transformed constraints, without recomputing it from the original constraints. Next, we show that applying $\log$-transformation preserves the maximizer of Problem~\ref{problem:lhf}.

\begin{theorem}\label{theorem:logtr}
    Consider Problem~\ref{problem:lhf} with an STL formula $\phi$ and a dataset $\dataset$. Assume $\dataset$ consists of signals that satisfy all predicates of $\phi$ at all time instances that are required to compute $\rob(\signal, \phi)$. Then, applying the $\log$-transform to the robustness constraints does not change the optimizer of Problem~\ref{problem:lhf}. 
\end{theorem}

\begin{proof}
The logarithm is a strictly increasing function, and therefore preserves the ordering of all $\min$ and $\max$ operations in the formula. That is, for any set of $a_1, \ldots, a_n \in \reals_{>0}$, we have $\log(\min(a_1, \ldots, a_n)) = \min(\log(a_1), \ldots, \log(a_n))$, and similarly for $\max$. Hence, while the numerical values of each operation in the \robustnesstree change, the optimizers remain unchanged.
\end{proof}

We now discuss how to relax the positivity assumption, i.e., handling signals whose robustness computation constraints contain non-positive values. Two cases need to be considered: signals with all negative values, and signals with mixed values.

For signals whose robustness computation contains only negative robustness values, e.g., signals which violate all predicates of a formula at all time instances needed for the robustness computation, we cannot use $\log$-transform directly. However, because all these values lie in the negative quadrant, we can still compare them using the logarithm of their absolute values. Therefore, we can separate the sign from the magnitude, take the logarithm of the absolute value and multiply it with negative before performing the \texttt{min/max} operations, obtaining equations such as $\min_{i \in \{1,2\}}(-(\log(w_i)+\log(|\wrob(\signal,\phi_i,t)|))$. Equivalently, this can be handled by reversing the operator: for example, taking $\max$ instead of $\min$ when the values are negative, and removing the negative sign.

For signals whose robustness computation contains a mix of positive and negative values, we leverage the structural pruning algorithm to unify the signs in the computation. Applying structural pruning to a signal's robustness computation allows us to remove any subtree that results in a robustness value whose sign is different than the sign of the robustness value at the root node. Therefore, pruning is a \emph{required} step for the $\log$-transform to be applicable, as it ensures that all remaining constraints have robustness values with the same sign. As in the all-positive case, we show that applying these two procedures sequentially in a learning-from-human-feedback problem does not change the maximizer of the problem.

\begin{theorem}
Consider Problem~\ref{problem:lhf} with an STL formula $\phi$ and a dataset $\dataset$. Applying structural pruning to the signals in $\dataset$ and the $\log$-transform to the robustness constraints does not change the optimizer of the problem.
\end{theorem}

\begin{proof}
The proof directly follows Theorems~\ref{theorem:pruning} and~\ref{theorem:logtr}. Structural pruning removes only subtrees that cannot influence the overall robustness of each signal. The $\log$-transform preserves the ordering of all $\min$ and $\max$ operations for the remaining positive or negative values. Therefore, the combination of pruning and logarithmic transformation does not change the optimizer of the original problem.
\end{proof}

With the structural pruning and log-transform procedures, we substantially reduce the problem size and convert the inherently multi-linear constraints into a linear form, yielding a computationally efficient method for optimally solving learning from human feedback problems. Specifically, the maximizer of Problem~\ref{problem:lhf} $\weight^*$, produces a WSTL formula $\phi_{\weight^{*}}$ whose qualitative semantics ensure that safety constraints are never violated—regardless of the learned weights. Furthermore, each learned weight in $\weight^*$ provides direct interpretability, as its value quantifies the relative importance of the corresponding subformula or time instance in determining the overall robustness. Thus, our approach not only guarantees safety but also facilitates transparent personalization by highlighting which task components most strongly reflect the user's preference. 

\section{Experiments}
To evaluate the performance of the proposed method, we focus on two experiments: safe preference learning for robot navigation tasks, and interpretable learning to rank using real-life Formula 1 data. The implementation details are available at \url{https://github.com/ruyakrgl/human2wstl.git}.

\subsection{Safe Preference Learning for Robot Navigation}
In this experiment, we consider a simple robot navigation task. The robot starts at $(x, y)=(1,1)$, must visit either region $A = [7,9]\times[1,3]$, or region $B = [1,3]\times[7,9]$, within the first $10$ seconds. It should then proceed to region $C=[7,9]\times[7,9]$ within the next $10$ seconds while always staying in the environment bounds $E = [0,10]\times[0,10]$, and avoiding the unsafe region $U=[3,6]\times[3,6]$. This task is represented by the STL formula $\phi = \F_{[0,10]}((x,y) \in A \lor (x,y) \in B) \land \F_{[10,20]}\G((x,y) \in C) \land \G_{[0,20]}\lnot((x,y) \in U) \land \G_{[0,20]}((x,y) \in E)$. The robot follows a double-integrator model: 
\begin{equation}
q[t+1] = \mat{1 & 1 & 0 & 0 \\ 0 & 1 & 0 & 0 \\ 0 & 0 & 1 & 1 \\ 0 & 0 & 0 & 1 }q + \mat{0\\ 1\\ 0\\ 1}u, \quad u = \mat{a_x \\ a_y},
\end{equation}
where the state $q[t] = \mat{x& v_x& y& v_y}^T$ represents robot's current position and velocity. We begin by randomly sampling weight valuations ten times to generate a trajectory for each of the WSTL formulas using PyTelo \cite{cardona2023flexible}, and form five trajectory pairs. From these, we construct three preference datasets: PD1, the original set of preferences; PD2, obtained by flipping the preference of a single pair in PD1; PD3, obtained by inverting all preferences in PD1. For each preference dataset, we solve the MILP problem to learn the weight valuations that best represent the given preferences, yielding formulas $\phi_{\weight}^{\text{PD}i}$. Finally, we synthesize trajectories that will maximize the \wstlrob of the corresponding $\phi_{\weight}^{\text{PD}i}$. Figure~\ref{fig: robots} shows the synthesized trajectories. The results demonstrate that our method is responsive to even small preference changes and captures this diversity by producing distinct task executions.

    \begin{figure}[hbt!]
        \centering
        \includegraphics[width=0.7\linewidth]{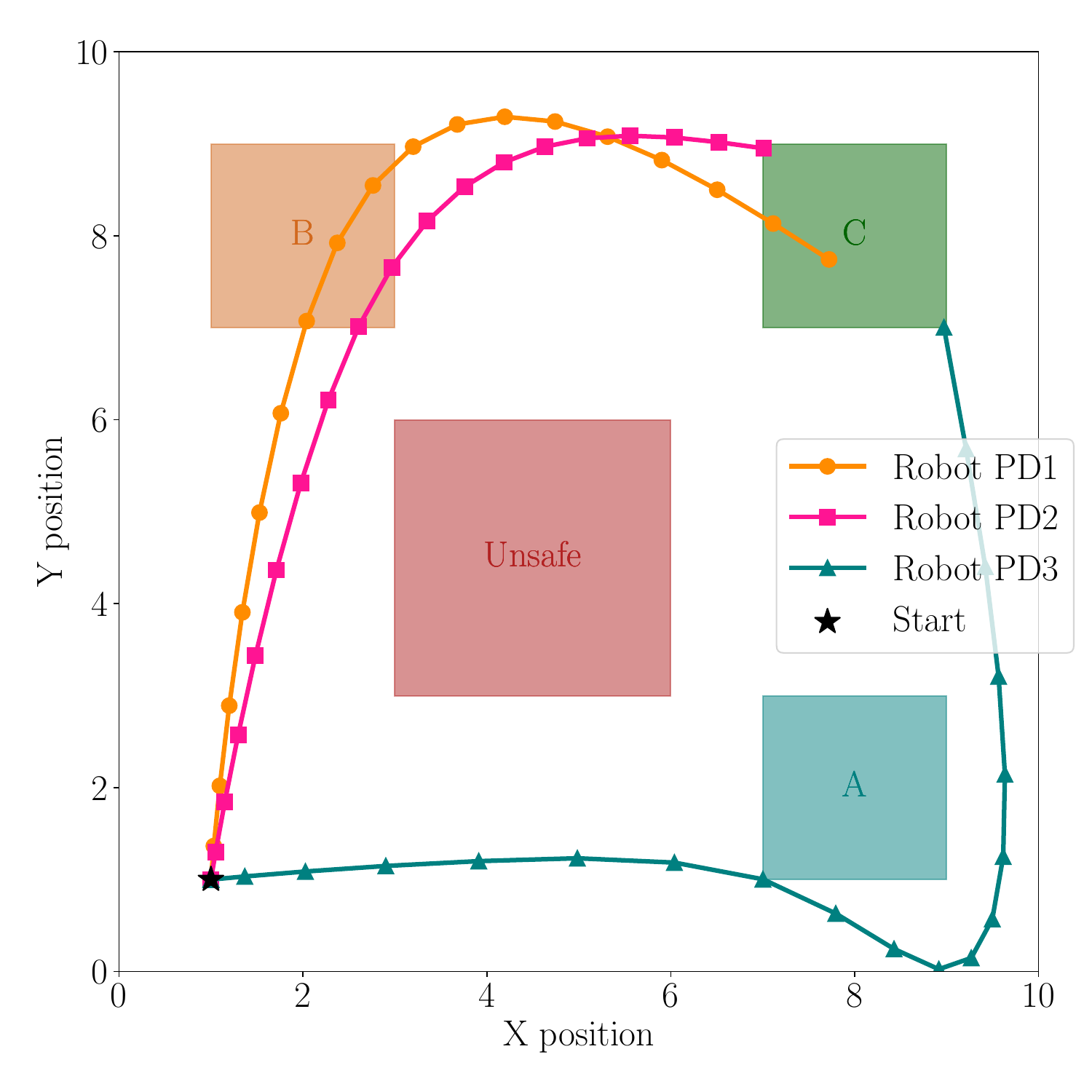}
        \vspace{-0.3cm}
        \caption{Trajectories generated using three different preference sets. PD1 is the original preference set, PD2 is obtained by flipping the answer to a single pair in PD1, and PD3 is obtained by reverting all answers in PD1.}
        \label{fig: robots}
        \vspace{-0.6cm}
    \end{figure}

\subsection{Learning to Rank}
For the learning-to-rank example, we consider a real-world motorsport setting: Formula 1. Formula 1 is a wheel-to-wheel racing championship where the goal is to finish in the highest position after a fixed number of laps. Drivers are ranked based on their finishing times, and each team, with two drivers per race, heavily relies on on-track data to determine their race strategies, such as pit stops, tire compound selection, overtake plans, etc. Our goal is to learn a WSTL formula that captures the important aspects of racing on a specified track using previous years’ race data and ranking results. The learned WSTL formula represents an ideal race performance. Note that, although this method is introduced as safe learning, this example shows that the approach can be extended to any kind of task execution guarantee written in STL.

A good race performance can be characterized by: starting from a competitive grid position, finishing the race, recording consistently fast lap times, outperforming (and overtaking) the car ahead, executing as few efficient pit stops as necessary, and avoiding on-track issues that affect the lap time, such as collisions and penalties \cite{choo2015racepredict}. To encode these principles in STL formalism, we consider several per-lap signals obtained via FastF1, a Python library for analyzing Formula 1 data \cite{oehrlyfastf1}. Let $N$ be the number of laps run in a race. The percentage of laps completed by a car is $r_{lap}\in [0, 1]^{N}$. If a car retires early or is lapped, $r_{lap}(k)<1$, for all laps $k \in \{1, \ldots, N\}$. The time-in-pit $t_{pit} \in \reals^{N}$ records the duration of a pit stop on each lap, and is computed from the car’s entry and exit times in the pit lane (zero if no pit stop occurs). The fuel-and-pit-corrected lap times $\bar{t}_{lap} \in \reals^{N}$ adjust the raw lap times for both pit-stop duration and the reduction in car mass due to fuel consumption, and is computed for each lap $k \in \{1, \ldots, N\}$ as $\bar{t}_{lap}(k) = t_{raw}(k)- (m_{full} - c_{kg/lap}k)s_{s/kg}-t_{pit}(k)$ where $m_{full}=100$ represents the mass of the fully fueled car, $c_{kg/lap}$ represents the fuel use per each lap, and $s_{s/kg} = 0.033$ represents mass sensitivity of lap time \cite{Heilmeier2018f1}. The position $p$ records the current race ranking of the car. The delta-to-lead $t_{\Delta lead}$ records the lap time differences between the ego and the car in front. Finally, the safety car presence $b_{SC} \in\{-1, 1\}^N$ is set to $1$ during the Safety Car and Yellow Flag periods (i.e., when cars reduce speed) and $-1$ otherwise.

Using these signals, we define the formula characterizing the performance as $\phi = \phi_{progress}\land \phi_{grid} \land \phi_{\Delta lead} \land \phi_{lap} \land \phi_{pit}$, where $\phi_{progress} = \F(r_{lap} \geq 0)$ represents the effect of the progress,
$\phi_{grid} = \G_{[0,1]}(p \leq 20)$ represents the effect of the starting position, $\phi_{\Delta lead} = \G(\lnot b_{SC} \implies t_{\Delta lead} \leq 30s)$ represents the effect of competitiveness, $\phi_{lap} = \G(\lnot b_{SC} \implies \bar{t}_{lap}\leq 150s)$ represents the effect of the lap times excluding the reduce-speed periods, and $\phi_{pit}=\G(t_{pit} > 0 \implies t_{pit}\leq 30s)$ represents the effect of having efficient pit stops. To compare and analyze the contribution of the formula, we also use $\tilde{\phi} = \phi_{progress}\land \phi_{grid} \land \phi_{\Delta lead} \land \phi_{pit}$. Note that $\tilde{\phi}$ is the same except $\phi_{lap}$ is not in the specification.

\begin{figure*}[hbt!]
    \centering
\begin{subfigure}[t]{0.49\linewidth}
        \centering
    \includegraphics[width=\linewidth]{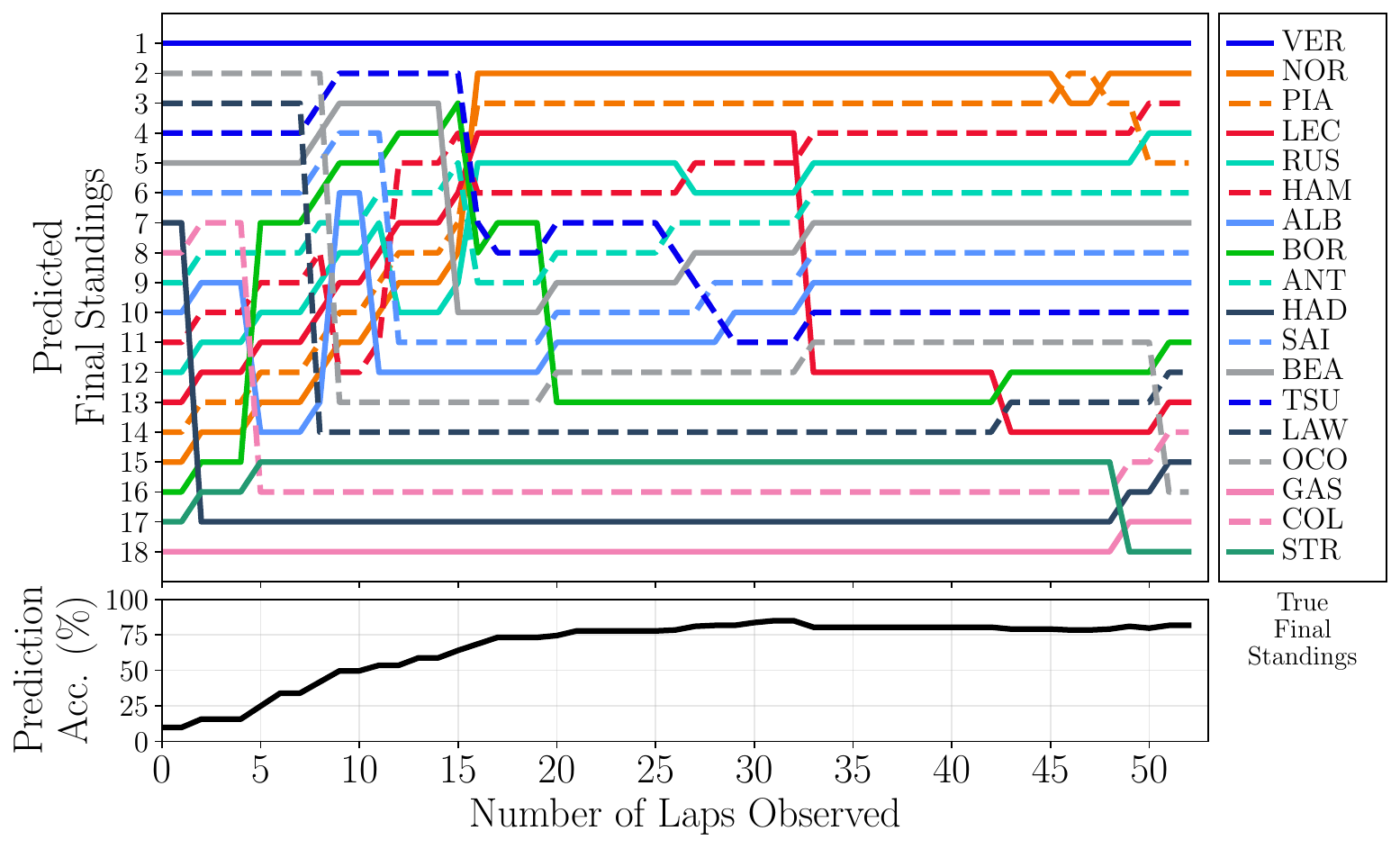}
    \caption{DNF/DNS excluded}
    \vspace{-0.2cm}
    \label{fig:dnf_excl}
\end{subfigure}
\begin{subfigure}[t]{0.49\linewidth}
        \centering
    \includegraphics[width=\linewidth]{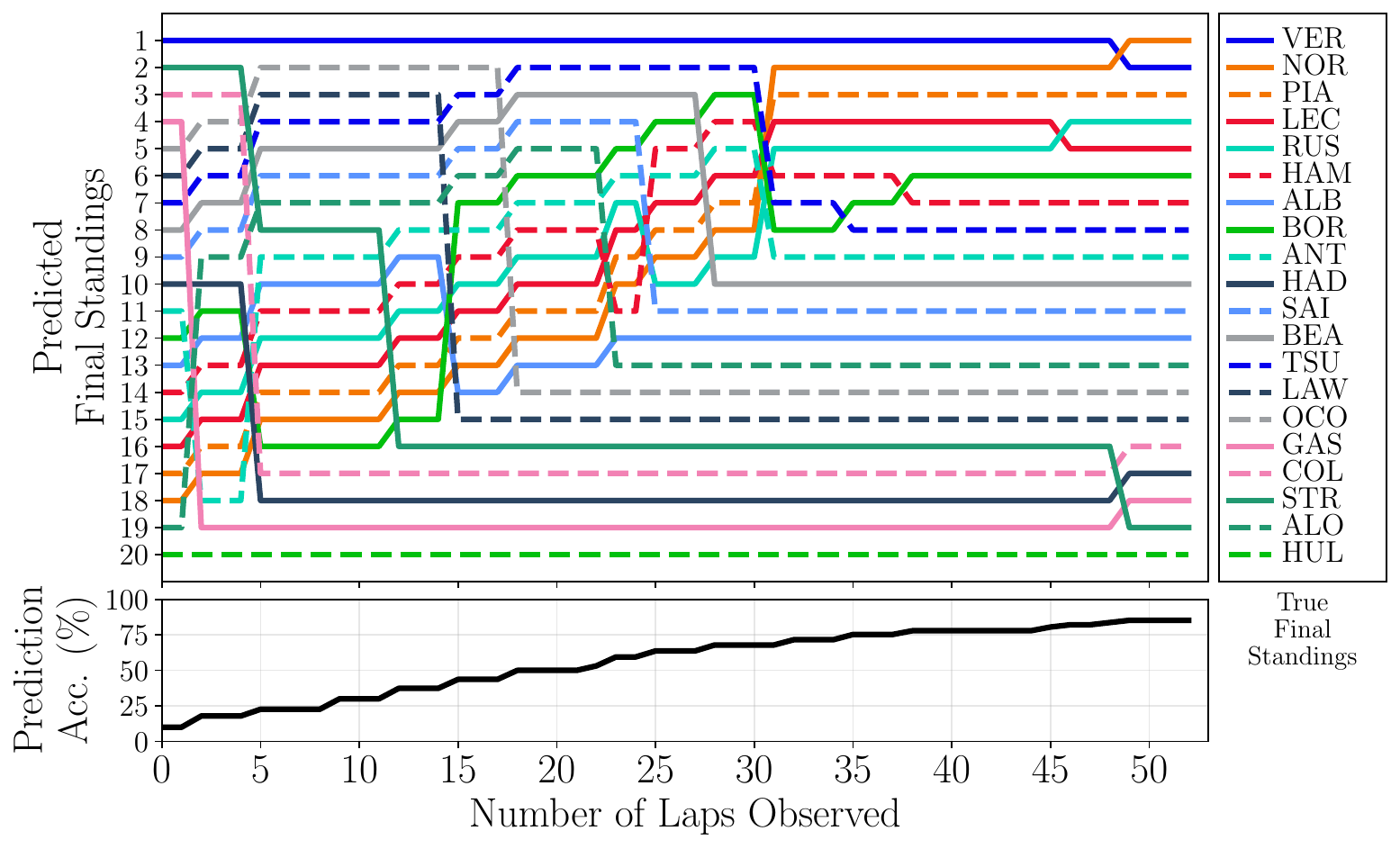}
    \caption{DNF/DNS included}
    \vspace{-0.2cm}
    \label{fig:dnf_incl}
\end{subfigure}
    \caption{The evolution of the final standing predictions at the 2025 Monza Grand Prix when (a) DNF and DNS are excluded, and (b) DNF and DNS are included. For each subfigure, the plot on the left shows the predictions after each lap, and the block on the right shows the true final standings. Each team is represented with a separate color, and drivers from the same team are represented with solid and dashed lines. For readability, driver abbreviations are used instead of enumerations.}
    \label{fig:f1_results}
    \vspace{-0.4cm}
\end{figure*}

To evaluate the accuracy of the ranking learner, we use the normalized Kendall-$\tau$ distance, which measures the percentage of discordant pairs, i.e., incorrectly-ordered pairs, in all pairwise comparisons $(i,j)$ where $i \in \{1, \ldots, P-1\}$ and $j \in \{i+1, \ldots, P\}$. That is, $$\text{acc}((\signal_1, \ldots, \signal_P)) = 1 - \frac{\text{number of discordant pairs}}{{\binom{P}{2}}}.$$

\textit{Experiment details}: To ensure consistent racing conditions over the years, such as weather and safety car appearances, we focus on the Monza Grand Prix. Since there are no major regulation changes in recent years, we use data from 2021, 2022, and 2024 for training, and the 2025 race for testing. The 2023 Monza Grand Prix is excluded due to a false start that shortened the race to $51$ laps instead of the usual $53$. For both training and test data, we scale all signals and thresholds at each year to $[0,1]$ by min-max normalization, to avoid biases due to the order-of-magnitude difference across channels. Because real races involve unpredictable events such as collisions and technical failures, we perform two sets of experiments: one excluding Did Not Finish (DNF) and Did Not Start (DNS) cars and one including them.

For both cases, we warm-start the MILP solver by selecting the best-accuracy solution among $S = 10000$ uniform random weight samples (cf. Random Sampling (RS) method of \cite{Karagulle2024spl}). 
We set the cost function as defined in Sec.~\ref{sec:pstament} with the regularization term $\lambda=0.2$. We then solve the MILP problem with a four-hour time limit, using Gurobi on a computer with 64 AMD Ryzen Threadripper PRO 5975WX 32-Cores and report the best solution found. As a baseline comparison, we repeat the RS method $1000$ times with $S_{RS} = 10000$ weight samples, and report the mean and standard deviations of accuracies. RS is shown to be comparable to Gradient-Descent methods in terms of accuracy, and more efficient in terms of computation time \cite{Karagulle2024spl}. We do not compare with the naive version (without pruning and log-transform) as MINLPs are highly costly, and the proposed method outperforms the MINLP problem by construction, i.e., a simpler and smaller problem. Table~\ref{tab:results} shows the results on training and test sets for both DNF options and both formulas. It is important to note that even though a global optimal is not guaranteed when the solver is time-limited, the MILP procedure improves over mean RS performance on the training set and shows promising performance on the test set. Also, when compared to its warm-started accuracy, our method raises the accuracy by up to $7\%$. Moreover,  when $\phi$ is used, the learned weights generalize to future seasons with different cars, different teams, and different drivers. This highlights the fact that, although driver skills and car performances are at the core of the race success, the proposed method can capture driver- and car-agnostic performance points. We do not observe the same performance with $\tilde{\phi}$, specifically, when DNF/DNS are included, the formula overfits to previous years. Options to fix are increasing the regularization constant or choosing a more representative formula.

When we take a closer look at the learned weights, we see that, for the DNF/DNS excluded case, the most important part of the formula is the initial grid position, followed by lap times, pit times, delta-to-lead, and completed laps. On the other hand, when DNF/DNS is included, the importance order changes to first the lap times, followed by the completed laps, initial grid position, and pit times, and finally, delta-to-lead. These insights can help teams improve their race-day strategies.

\setlength{\tabcolsep}{3.7pt} 
\setlength{\aboverulesep}{1pt}
\setlength{\belowrulesep}{1pt}
\begin{table}[h!]
\caption{Performance of RS and the proposed method on training and test sets in terms of accuracy (\%) for $\phi$ and $\tilde{\phi}$.}
\label{tab:results}
\begin{tabular}{cccccc}
\toprule
  \multirow{2}{*}{Formula} & \multirow{2}{*}{Method}  & \multicolumn{2}{c}{DNF/DNS excl.} & \multicolumn{2}{c}{DNF/DNS incl.}\\
   &  & Train & Test & Train & Test \\\midrule
  \multirow{2}{*}{$\phi$} & RS\cite{Karagulle2024spl} & $86.7\pm0.6$ & $82.5\pm2.6$& $81.2\pm0.8$& $78.8\pm10.4$  \\
   &  Ours & $92.9$ & $81.7$  & $85.3$ & $85.3$ \\\midrule
  \multirow{2}{*}{$\tilde{\phi}$} & RS\cite{Karagulle2024spl} & $86.7\pm0.7$ & $82.5\pm2.5$& $81.3\pm0.8$& $78.9\pm10.0$  \\
   & Ours & $93.9$  & $81.7$ & $90.2$  & $64.74$ \\    
  \bottomrule
\end{tabular}
\end{table}

Finally, we want to answer the following question: given only the laps observed thus far until lap $K$, how accurately can we predict the final standings? Figure~\ref{fig:f1_results} shows the evolution of the prediction accuracy of $\phi$ over the race. More specifically, when DNF/DNS are excluded, as Figure~\ref{fig:dnf_excl} shows, the model reaches over $85\%$ accuracy after fifteen laps. On the other hand, when DNF/DNS are included, the accuracy improves more slowly over laps, reflecting the unpredictable nature of the race events.

It is worth emphasizing that the goal of this experiment is not to provide insights into Formula 1 but to serve as a proof of concept. The results highlight the abilities of temporal logic learning in analyzing and learning patterns in an interpretable and efficient manner for complex problems like learning to rank on real-world sequential data, such as racing.  

\section{Conclusion, Limitations and Future Work}
In this work, we proposed a safe learning-from-human-feedback framework, based on temporal logic formalism, that can be efficiently solved to optimality. The framework consists of two key procedures: structural pruning and $\log$-transform. Structural pruning systematically removes parts of the robustness tree that do not affect the robustness value, reducing the problem size. The $\log$-transform linearizes the robustness computation constraints with respect to weights, enabling an MILP formulation. We demonstrated the performance of the proposed method on two experiments: the preference learning problem for a robot navigation task and the learning-to-rank problem using Formula 1 data. Results show that the method can capture diverse preferences effectively.

Despite the success of the proposed method in the experiments, several factors may limit its applicability. Mainly, the method requires domain knowledge and expertise in writing temporal logic formulas to define task specifications, as well as careful hyperparameter tuning to avoid overfitting. While this need introduces a reliance on expert input, it comes with the benefit of interpretability. Looking forward, we plan to integrate more streamlined task specification methods, potentially using large language models to translate natural language descriptions to STL formulas. 

\balance

\bibliographystyle{IEEEtran}
\bibliography{IEEEabrv,references.bib}
\end{document}